\documentclass{article}

\usepackage{kr}

\usepackage{times}
\usepackage{soul}
\usepackage{url}
\usepackage[hidelinks]{hyperref}
\usepackage[utf8]{inputenc}
\usepackage[small]{caption}
\usepackage{graphicx}
\usepackage{amsmath}
\usepackage{amsthm}
\usepackage{booktabs}
\usepackage{enumerate}
\usepackage{xspace}
 \usepackage{todonotes}
\usepackage{amssymb}
\usepackage{amsmath}
\usepackage{amsthm}
\usepackage{mathtools}
 \presetkeys{todonotes}{inline}{}

\usepackage{algorithm}
\usepackage[noend]{algpseudocode}
\algnewcommand{\LineComment}[1]{\State \(\triangleright\) #1}

 \newcommand{\jonas}[1]{{\textcolor{blue}{Jonas: #1}}}
 \newcommand{\niklas}[1]{{\textcolor{orange}{Niklas: #1}}}
  \newcommand{\hector}[1]{{\textcolor{red}{Hector: #1}}}

\newtheorem{theorem}{Theorem}
\newtheorem{definition}{Definition}

\usepackage{algorithm,algpseudocode}

\newcommand{\Omit}[1]{}
\newcommand{\tup}[1]{\langle #1 \rangle}

\newcommand{\citeay}[1]{\citeauthor{#1} (\citeyear{#1})}

\usepackage{lipsum}

\newcommand{\sift}{\textnormal{\textsc{sift}}\xspace}
\newcommand{\synth}{\textnormal{\textsc{synth}}\xspace}
\newcommand{\strips}{\textnormal{\textsc{strips}}\xspace}
\newcommand{\stripsp}{\textnormal{\textsc{strips+}}\xspace}

\newcommand{\navig}{\textsc{navig}\xspace}
\newcommand{\puzzle}{\textsc{sliding-tile puzzle}\xspace}
\newcommand{\down}{\textsc{down}}
\newcommand{\up}{\textsc{up}}
\renewcommand{\left}{\textsc{left}}
\renewcommand{\right}{\textsc{right}}

\title{Learning Lifted Action Models From Traces of Incomplete Actions and States}

\author{
Niklas Jansen \and
Jonas Gösgens \and
Hector Geffner \\
\affiliations
RWTH Aachen University\\
  {\emails \{niklas.jansen, jonas.goesgens, hector.geffner\}@ml.rwth-aachen.de}\\
}

\begin{document}

\maketitle
\begin{abstract}
  Consider the problem of learning a lifted \strips model of the sliding-tile puzzle
  from random state-action traces where the states represent the location of the tiles only,
  and the actions are the labels up, down, left, and right, with no arguments.
  Two challenges are involved in this  problem.
  First, the states are not full \strips states,   as some predicates are  missing, like
  the atoms representing the  position of the  ``blank''.   Second, the actions are not full  \strips either,
  as they do not  reveal all the  objects  involved in   the actions  effects and preconditions.
  Previous approaches have addressed different
  versions of this model learning problem,  but most  assume  that
  actions in the traces are full \strips actions or that the domain predicates
  are all observable. The  new setting  considered in this work is more ``realistic'', as 
  the atoms observed convey the state of the world but  not full \strips states, and
  the actions reveal the arguments  needed for selecting the action but
  not the ones needed for modeling it in \strips.   For formulating and  addressing
  the  learning problem,  we introduce a variant of \strips,   which we call \stripsp, 
  where certain  \strips action arguments can be   left implicit in   preconditions which  can also
  involve a limited form of existential  quantification. The learning problem  becomes  the problem
  of learning \stripsp models   from \stripsp state-action traces. For this, the  proposed learning algorithm, called
  \synth,   constructs  a stratified sequence (conjunction) of precondition expressions or ``queries'' for each action, 
  that denote  unique objects  in the state  and ground the implicit action arguments in \stripsp.
  The correctness and   completeness of \synth  is established, and its   scalability is tested
  on state-action traces obtained from  \stripsp models derived from existing \strips domains.
\end{abstract}

\section{Introduction}


The problem of learning action models from data is fundamental in  both planning and reinforcement learning. 
In classical planning, lifted  models are  learned from observed  traces that may contain
action and state information from a  hidden \strips domain
\cite{yang2007learning,arora2018review,aineto2019learning,balyo:2024}, 
while in  model-based reinforcement learning 
action models  are learned from similar traces but without making assumptions about the structure of the hidden
domain  \cite{sutton:book,brafman:rmax}. As a result, model-based RL approaches have been successfully used in non-\strips domains like the
Atari games \cite{fleuret:atari,dreamer:atari,timofte:atari} but the models learned are not lifted nor   transparent.

In this paper, we aim to start exploring the middle ground between model-learning approaches in classical planning
and  reinforcement learning. A key step for this is to drop the assumption that the observed actions or states in the input
traces  come from a hidden \strips domain. This is because such types of traces make unrealistic assumptions about the information
that the  learning agent can actually perceive. Consider  for example the   domains \navig and \puzzle. 
In the first, an agent can move in  a  grid, one unit at a time;  in the other,  the ``blank
tile''  can move in a grid in  the same way.  The observed actions in the two domains can be  described uniquely
in terms of the four labels \up, \down, \left, and \right, yet these labels do not represent full \strips actions.
The reason is that actions in \strips are forced to include as arguments \emph{all the objects
involved in atoms that change}. This means that in \navig, the actions 
must  include the current and next location of the agent as arguments,
while in the \puzzle, they must include the tile that is moved, along with its current
and new location. 

While these  examples illustrate that  it is not reasonable to assume that the observed traces contain full \strips actions, 
a similar argument can be made about  full \strips states. Indeed, \strips\ encodings of the  \puzzle need to include a predicate
tracking the position of the ``blank'', but this predicate is not needed to represent the state of the world which is given by
the positions of the tiles alone.

An alternative to traces with full \strips actions and/or full \strips states is to make both actions and states
\emph{partially observable} \cite{paolo:2025}. Yet without futher restrictions, the learning problem is not well defined and
may   have no general solution.
Our goals in this paper are more modest but they can be understood from this perspective as well.
We consider a specific version of this general learning problem where certain \strips action arguments
and predicates  are fully observable, while the other action arguments and predicates are not observable at all.
Yet  from these partial observations we will be able to uncover the lifted \strips model in full.

For formulating and  addressing   the resulting model learning problem we introduce a variant of the \strips language,
which we call \stripsp,   where certain  \strips action arguments can be  left implicit in  preconditions
which can also  involve a limited form of existential   quantification. The learning problem becomes the problem of
learning \stripsp models   from \stripsp state-action traces, and the proposed learning algorithm will involve
the construction of a sequence of precondition expressions or ``queries'' for each action, each one built from the
observed predicates  and denoting   unique objects  in the  states where the action is executed.  The  correctness and
completeness of the learning algorithm will be  established, and its scalability will be  tested
on state-action traces obtained from  \stripsp models obtained from existing \strips domains.

The rest of the paper is structured as follows. We review  related work and the \strips language first, then
introduce the new target language for learning, \stripsp, and the model learning task and  algorithm.
Finally, we present the experimental results,   conclusions,  and  challenges.

\section{Related Work}


\medskip

\noindent \textbf{Learning from states and actions.}
The problem of learning lifted \strips  models has a long history in planning. In most cases, the input 
traces combine information about  actions and states. While observability of the states can be  partial or  noisy, 
in almost all cases the observations reveal all the domain predicates and their arities
\cite{yang2007learning,mourao2012learning,zhuo2013action,arora2018review,aineto2019learning,lamanna2021online,verma2021asking,macq:muise,le2024learning,bachor2024learning,xi2024neuro,aineto2024action,icaps:learning}. Likewise, the actions  are normally   full \strips actions
with all the arguments, the exception being a recent SAT-based learning
formulation where only the action names are observed  (along with the states),
with no information about either their arguments or their arity \cite{balyo:2024}.
Interestingly, \citeay{paolo:2025} recently considered  learning action models
from state-action traces where both the states and the action arguments can be partially observable,
yet   the approach  comes with no  guarantees. Indeed, the formulation, which is based on observations
 about the atoms  affected by an action,  implies that action arguments that are just involved
 in preconditions and    are not affected by the action,  must be observable, at least sometimes;
 else such   arguments cannot be recovered.\footnote{As  a  concrete example, a correct model for an  action like  $move(p,t_1,t_2,l)$
which  moves  a package $p$ from truck $t_1$ to truck $t_2$ when  both trucks are at location $l$
and package $p$ is in truck $t_1$, cannot be learned in this approach if the action argument $l$ is not observable, 
as the location $l$ cannot  be identified from the effects of the action alone. In \synth, this is not a problem.}
%

\medskip

\noindent \textbf{Learning from actions only.}  Fewer works have considered the problem of learning lifted \strips models from traces containing actions only.
The \textsc{locm}  systems \cite{locm1,locm,locm2,locm3}  accept action traces as inputs,  and outputs lifted domain descriptions, but their properties
and scope are not clear. More recently, the \sift algorithm, which  uses the same input traces as \textsc{locm}, has been shown to be sound, complete, and
scalable \cite{sift}. The problem with \sift\ is that the actions in the traces are assumed to be full \strips actions with all the action arguments
spelled out. 
There is also a SAT approach to lifted model learning, that accepts state graphs, not traces, 
where the states are not observable  and  edges are labeled with  action names and no arguments \cite{bonet:ecai2020,ivan:kr2021}.
This approach learns  from  very sparse information, but unlike \sift and \textsc{locm},  does not scale up. 

\medskip

\noindent \textbf{Model-based reinforcement learning.}   Model-based RL algorithms learn controllers by learning (stochastic) models first,
without making further assumptions about the structure of the  models \cite{sutton:book}. In the tabular setting, they  result in  flat state models
with  state transition probabilities obtained from simple counts  \cite{brafman:rmax}. In some cases, a first-order state language  is assumed but the state predicates
are given \cite{littman:rmax,leslie:probabilistic}. In more recent approaches, the learned dynamics is not represented compactly in languages such as \strips or
PDDL, but in terms of deep neural networks.  In particular, successful model-based RL approaches for the Atari games
or Minecraft learn  the action dynamics  using transformers or recurrent neural networks \cite{fleuret:atari,dreamer:atari,timofte:atari}.
A limitation of these methods, like other recent deep-learning approaches that learn \strips models from state images \cite{asai:latplan,asai:jair},
is that the learned action models are opaque and not lifted.

\section{Background: \strips}


A classical \strips planning problem is a pair $P=\tup{D,I}$ where $D$ is a first-order
\emph{domain} and $I$ contains information about the instance \cite{geffner:book,ghallab:book}.
The domain $D$ has a set of predicate symbols $p$ and a set of (lifted)  action schemas $a(x)$ with
preconditions, add, and delete  effects $Pre(a(x))$, $Add(a(x))$, $Del(a(x))$ 
given by  atoms $p(x_1, \ldots, x_k)$, where $p$ is domain predicate
of arity $k$, and each $x_i$ is an argument of the action schema. 
The instance information is a tuple $I=\tup{O, \textit{Init},G}$ where $O$ is a set of object names or constants  $c_i$, and
$\textit{Init}$ and $G$ are sets of \emph{ground atoms} $p(c_1, \ldots, c_k)$ denoting the initial and goal situations.

An action schema $a(x)$, where $x$ is a tuple of variables $x_1, \ldots, x_{n}$,  is instantiated by consistently replacing
the variables $x_i$  by constants $c_i$ in the instance. In the typed version of \strips, the variables and the
constants have types, and variables are replaced by constant of the same type. The atoms  $p(x)$ are called fluent
if they appear in the effect of some action, and else they are called static. Static atoms    appear   in
the initial situation and action preconditions, and affect just the action groundings.

A \strips  problem $P=\tup{D,I}$ defines a state model $S(P)=\tup{S,s_0,S_G,\textit{Act},A,f}$
in compact form where the states $s \in S$ are \emph{sets of ground atoms} from $P$ (assumed to be the true
atoms in $s$), $s_0$ is the initial state $I$, $S_G$ is the set of goal states $s$ such that $G \subseteq s$,
$\textit{Act}$ is the set of ground actions in $P$, $A(s)$ is the set of ground actions
whose preconditions are (true) in $s$, and $f(a,s)$, for $a \in A(s)$, represents
the state $s'$ that follows action $a$ in the state $s$; namely $s' = ((s \setminus Del(a)) \cup Add(a))$.

By design, a  \strips action can only  affect the truth of an atom $p(x_1, \ldots, x_k)$ if the atom
is an  action effect, and hence if all of  its  arguments $x_i$ are action arguments. This implies
for example, that if an  action for picking up a block $x$ makes an atom $clear(y)$ true, then
$y$ must be an action argument.


An action sequence $\tau=a_0,a_1,\ldots,a_{n}$ is applicable in $P$ if
$a_i \in A(s_i)$ and $s_{i+1}=f(a_i,s_i)$, for $i=0,\ldots,n$. The states
$s_i$ are said to be reachable in $P$, and the action sequence $\tau$
is a plan for $P$ if $s_{n+1}$ is a goal state. 

An action trace in  $D$ is an  applicable action sequence in some instance
$P=\tup{D,I}$. Algorithms like \textsc{locm} and \sift learn  (lifted) \strips models
from  action traces alone. Others approaches learn from traces $s_0, a_0, s_1, \ldots$
that combine states and actions, with some information about actions or states missing,
or corrupted by noise.  In our setting,  \strips models will be learned from  state-action traces
where some of the action arguments and some of the predicates  in the state are not observable at all.

\section{\stripsp}

We cast the problem of learning lifted \strips models from traces of incomplete \strips actions and states into
the problem of learning  models in a language that is more succint  than \strips, that we call
\stripsp.  \stripsp extends \strips  by allowing \emph{tuples  of free variables}  $y$ and $z$ in the action schemas $a(x)$ that are not
among the explicit  action arguments $x$. The  variables in $y$   can only appear in action preconditions, 
while the variables in $z$  can appear in action preconditions and effects. The $z$ variables,
however, have to be \emph{determined} by the $x$ variables as spelled out below.
This  extension of  \strips  is  not  new and forms  part of some of  the PDDL standards \cite{pddl:book}.
In particular,  the first PDDL standard \cite{pddl} supports the $z$ variables,
which are    declared via the keyword \texttt{:vars}.

\begin{definition}
  Action schemas $a(x)$ in \stripsp have   (conjunctive)  preconditions $Pre(a(x))=\phi(x,y,z)$ with free variables
  among those of  $x$, $y$, and $z$ which are  pairwise disjoint sets of variables. The variables in $x$
  and $z$ can appear in action effects. The value of the variables in $z$ must be \emph{determined}
  by the values of the variables in $x$ as defined below.
  \label{def:determined}
\end{definition}

The variables in $x$  denote explicit action arguments of the action $a(x)$  in \stripsp, and the variables
in $z$  denote ``implicit'' action arguments captured by the preconditions. \strips actions are trivial \stripsp actions with an empty list of implicit action
arguments; while \stripsp actions   map into \strips actions with more arguments after pushing the $y$ and $z$ variables into explicit action arguments in \strips. 
For this translation and semantics to be valid, 
the value (denotation)  of the $z$ variables that can appear in the action effects,
must be \emph{uniquely determined} by the value of the explicit action arguments $x$
and the action preconditions. 

For example, in the \puzzle, the action $up(c_1,c_2,t)$ that moves tile $t$ from cell $c_2$ to cell $c_1$
can be modeled in \strips with these three arguments. In \stripsp, on the other hand,
the action $up$  can be modeled without \emph{any} arguments, as the three explicit arguments in \strips, $c_1$, $c_2$, and $t$
can be recovered from the values of the three  free $z$  variables $z=\{z_1,z_2,z_3\}$ in its precondition $\phi(x,y,z)$
given by the  formula $blank(z_1) \land above(z_1, z_2) \land at(z_3,z_2)$.
Indeed, this precondition  is satisfied in each state $s$ where the action $up$ is applicable by a unique
grounding of the $z_i$ variables, so that such variables can be regarded  as \emph{implicit arguments} of the
$up$ action.  This is all  formalized below.

\medskip

A ground \stripsp action $a(o)$ is \emph{applicable}  in a state $s$ if its precondition formula $\phi(x,y,z)$ is \emph{satisfiable} in $s$
with a grounding that binds $x$ to $o$. Formally: 

\begin{definition}
  For a \strips or \stripsp domain $D$, 
  let $\phi(x,y,z)$ refer to a conjunction of domain  atoms with  arguments from  the three disjoint variable sets
  $x, y, z$ and  let $s$ be the  initial  state  of an  instance $P=\tup{D,I}$ of $D$.\footnote{Notice that a reachable
  state of an instance $P$ is the initial state of another instance $P'$ that  is otherwise  like $P$.}
  Then,
  \begin{itemize}
  \item A grounding of  $\phi(x,y,z)$ in $P$ is an assigment $\sigma$ of variables in the formula to constants (objects) in the instance. 
  \item A grounding \emph{satisfies}  the  formula $\phi(x,y,z)$ in $s$ if the resulting ground atoms are all true in $s$.
  \item The formula $\phi(x,y,z)$ is \emph{satisfiable} in $s$ if some grounding of the variables satisfies it.
  \end{itemize}
  \end{definition}

The $z$ variables are \emph{determined}  by the variables $x$ in a \stripsp  action $a(x)$ with precondition $\phi(x,y,z)$ if
the groundings that satisfy the formula must agree on the value (grounding) of $z$ when  they agree on the value of $x$: 

\begin{definition}
  The value of the  variables in  $z$ are \emph{determined} by the values of the  variables in  $x$ in the precondition  $\phi(x,y,z)$ of an \stripsp  action $a(x)$ in
  a domain $D$,   if in the initial state of any   instance $P$ of $D$, there are no two satisfying  groundings $\sigma$ and $\sigma'$
  of $\phi(x,y,z)$   such that $\sigma(x)=\sigma'(x)$ and $\sigma(z) \not= \sigma'(z)$.
\end{definition}

In the precondition $\phi(x,y,z)$ of the action $up$ above
given by the  formula $blank(z_1) \land above(z_1, z_2) \land at(z_3,z_2)$,
the three variables $z_i$ in  $z=\{z_1,z_2,z_3\}$ are determined,
as in any legal state $s$ over the sliding-puzzle domain $D$,
there is a unique grounding  for $z_1$ (blank position),
for $z_2$ (cell above the blank), and for  $z_3$ (tile at the cell which is above the blank).

In addition to assuming that the grounding of the $z$ variables in \stripsp action
schemas $a(x)$ is uniquely determined by the grounding of the explicit action arguments $x$
and the preconditions as expressed in Definition~\ref{def:determined}, 
we assume, as in \sift,  that action   effects are ``well-formed'' in the sense that they
change the state; namely, the complement of the effects must be explicit or implicit action preconditions,
so that  no action adds an atom that is  true, or deletes an atom that is   false \cite{sift}.

\medskip

Determined variables   express  \emph{state-invariants} of the domain; namely, that in the
reachable states where a ground action $a(o)$ of the lifted action $a(x)$ 
with  precondition  $\phi(x,y,z)$ applies,  there is a unique grounding for $z$
among the (non-empty) groundings that satisfy  $\phi(x,y,z)$ with  $x=o$.
In other words, the grounding of  $z$ is  a function  $f_{a,s}(x)$ of $x$, the state $s$, and
the action instance $a=a(o)$.

The \emph{semantics} of the \stripsp action $a(x)$ with precondition $\phi(x,y,z)$
is   the semantics of the \strips action $a'(x')$ that has the same preconditions and effects as $a(x)$
but with the $y$ and $z$ variables pushed as explicit arguments in $x'$. 

\Omit{
 In particular, a ground action $a(o)$ is applicable in  $s$ iff
 some ground instance $a'(o')$ of  $a'(x')$ is applicable in $s$ with $x'$ set to $o$, 
 and the effects of $a(o)$ are the effects of any such $a(o')$. Notice that
 $y$ variables can  appear in the preconditions of  $a(x)$ but not in its  effects,
 and hence a ground instance $a(o)$ of $a(x)$ may map into many applicable instances
 of the \strips actions  $a(o')$ that have different grounded preconditions
 but the same grounded effects.\footnote{Logically, the formula $\forall x, y, z. [\phi(x,y,z) \Rightarrow \psi(x,z)]$,
 is logically equivalent to $\forall x.z [(\exists y. \phi(x,y,z)) \Rightarrow \psi(x,z)]$, and this is why
 $y$ and $z$ variables in \stripsp can be regarded implicitly quantified existentially and universally, even
 if the grounding for $z$ that satisfies the antecedent $\phi(x,y,z)$ is a function of $x$ \cite{pddl}.}
}

\medskip

\noindent \textbf{Example:} \navig. The problem of navigating in an  empty grid can be modeled in \strips  via action schemas like $up(c,c')$,  where $c$ and $c'$
are grid cell variables,   preconditions are  $at(c)$ and $above(c',c)$, and effects are $at(c')$ and $\neg at(c)$.
In \stripsp, the explicit  action arguments $c$ and $c'$ can be made implicit through the use of the  preconditions  $at(z)$ and $above(z',z)$
where the variables $z$ and $z'$ are determined, as in every state $s$, there is a single atom $at(z)$ that is true, and  a single (static) true atom
$above(z',z)$ given $z$. The result is that the action $up$ can be modeled with no (explicit) arguments in \stripsp.

\medskip

\noindent \textbf{Example:} \puzzle.
In \strips, the domain  can be modeled by  actions like $up(c,c',t)$ where $t$ is the tile that moves from cell $c'$ to cell $c$,
and preconditions  involving   the  atoms $at(t,c')$, $blank(c)$, and $above(c,c')$, where $blank$ tracks the position of the ``blank''. In \stripsp,
the three action arguments can be made implicit as shown above.


\medskip

\noindent \textbf{Example:} \textsc{Blocks.} 
In \strips,  the action $unstack(x_1,x_2)$ takes as arguments the blocks $x_1,x_2$, where $x_1$ is stacked on $x_2$.
In \stripsp, the action $unstack(x_1)$ can take  instead a single explicit action argument $x_1$, as   the variable  $x_2$
can be captured by an  implicit variable $z_1$ whose unique grounding is determined by the value of
$x_1$ and the precondition $on(x_1,z_1)$ in any state where the action $unstack(x_1)$ is applied.
The explicit argument $x_1$ cannot be rendered implicit because multiple blocks may be  potentially unstacked.
Yet the action $unstack(x_2)$, where $x_2$ denotes the block beneath the one to be picked
would be a well-formed schema too,  with the atom $on(z_1,x_2)$ in its precondition
determining  the unique  grounding of the block $z_1$ to be unstacked.

\section{The Learning Task}

The learning task is to infer a   lifted \stripsp domain   from random state-action  traces obtained from instances
over a hidden \stripsp domain.  We   state the task formally   below after introducing some restrictions on the class of  hidden \stripsp models.

\subsection{Target Fragment of \stripsp}

We cannot address the learning task  in its full generality because  determining whether a precondition formula $\phi(x,y,z)$ is
satisfiable in a state is already NP-hard. We assume  instead  a class of hidden \stripsp domains $D$
whose precondition formulas  $\phi(x,y,z)$ are \emph{easy to evaluate}.  For convenience, we will refer to
formulas $\phi(x,y,z)$ as \emph{conjunctive queries}. In these formulas the atom predicates are the domain predicates
and the arguments are  free variables from  $x$, $y$, and $z$.\\
The restrictions below limit the expressive power of the target language, but every \strips problem is  part of this \stripsp fragment, as 
every \strips problem $P=\tup{D,I}$ is a  \stripsp problem with precondition  $\phi(x,y,z)$, 
where the sets of variables $y$ and $z$ are empty.

The first restriction applies to the $y$ variables: 

\begin{definition}
 A conjunctive query $Q(x,y,z)$ is \emph{simple} if each $y_i$ variable appears only once in $Q(x,y,z)$.
\end{definition}

This means that the only constraints on $y$ variables are those occurring in single atoms, and moreover,
 no variable $y$ appears twice in such  atoms either (this last condition could   be  relaxed though).
 The second restriction, \emph{stratification}, is more interesting and applies to the $z$ variables:
 
\begin{definition}
  A conjunctive query $Q(x,y,z)$   is \emph{stratified} if it can be expressed  as the conjunction of
  conjunctive queries  $Q_1(x,y,z^1)$, \ldots, $Q_n(x,y,z^n)$, each with one or more atoms, such that:
  \begin{itemize}
   \item Each variable $z_i$ appears in $z^i$ but  not in  $z^j$,  $j < i$.
   \item If $Q(x,y,z)$ is satisfiable in a state $s$ with grounding $\sigma(x)\!=\!o$ and $\sigma(z_i)\!=\!c_i$, $i=1, \ldots, n$,
     there is no grounding $\sigma'$ satisfying the prefix  $Q_1(x,y,z^1)$, \ldots, $Q_i(x,y,z^i)$ with $\sigma'(x) \!=\!o$ and
     $\sigma'(z_j) \not=\sigma(z_j)$,  for any $1 \leq j \leq i$, $i=1, \ldots, n$, 
  \end{itemize}
  \label{def:strat}
\end{definition}

  In other words, the query $Q(x,y,z)$ is \emph{stratified}, not just if  the $z$ variables are determined by the $x$ variables through the
  $Q$-formula, but if the value of each individual  variable $z_{i+1}$ in $z$ is determined by the value of the variables $x \cup \{z_1,\ldots, z_i\}$
  in the subformula $Q_{i+1}(x,y,z^{i+1})$. In fact, the variables $z$ being determined by $x$ means that if the query is satisfiable in a state $s$, there
  is a \emph{unique  value} $z= f_{a,s}(x)$  for the $z$ variables that satisfies  the query for a given grounding of $x$, state $s$, and action instance  $a=a(o)$.
  Computing these values, however, can be computationally hard. Stratification provides conditions under which this task is easy
  and can be solved \emph{one variable $z_{i}$ at a time.}

  A domain $D$ is  said to be \emph{stratified} if the action precondition formulas $\phi(x,y,z)$ are simple and
  stratified.

\subsection{Task}

The learning task is to infer a lifted \stripsp domain  $D_L$  from random state-action  traces obtained from  instances $P=\tup{D,I}$
of a hidden \stripsp domain $D$.  The learned domain $D_L$ does not have to be syntactically equivalent to the hidden domain $D$ but
has to be semantically equivalent,  meaning that the traces from  $P=\tup{D,I}$ must be  traces of  $P_L = \tup{D_L,I}$ and vice versa.

\begin{definition}[Learning task]
  Given state-action traces of the form $s_0,a_0,s_1,a_2,\ldots, s_n,a_n$ from \stripsp instances $P=\tup{D,I}$ of a 
  hidden \emph{stratified} domain $D$, the task is to learn a domain $D_L$ such that the instances $P'=\tup{D_L,I}$
  generate the same traces as $P$. 
\end{definition}

Thus, the domain $D_L$ is to be learned from traces obtained from instances $P=\tup{D,I}$ of a stratified domain $D$,
and $D_L$ and $D$ are deemed equivalent if  instances $P_i=\tup{D,I_i}$ and  $P'_i=\tup{D_L,I_i}$ generate the same
traces. 


\section{The Learning Algorithm}

We address  the learning task in three parts: learning the  preconditions   $Q(x,y,z)$ of the actions $a(x)$
that bind the  $z$ variables  uniquely, learning the extra preconditions $Q'(x,y,z)$ that use these bindings
but which do not constraint further  their  values, and  learning the action effects.

\subsection{Learning the Binding Preconditions $Q(x,y,z)$}
    The assumption that hidden  action  preconditions  $\phi(x,y,z)$ 
    are  simple and  stratified suggests  a simple  algorithm for learning a  formula equivalent to  $\phi(x,y,z)$
    from the traces  in two parts: a  ``conjunctive query''   $Q(x,y,z)$ that  binds the values of the $z$ variables, 
    and a second part $Q'(x,y,z)$ that uses such bindings. 
    The   query  $Q(x,y,z)$ is made up itself  of  conjunctive subqueries   $Q_1(x,y,z^1)$, \ldots, $Q_n(x,y,z^n)$ such that:
    \begin{enumerate}
    \item  \textbf{Stratification}: The formulas $Q_1(x,y,z^1)$, \ldots, $Q_n(x,y,z^n)$ stratify $Q(x,y,z)$  (Definition~\ref{def:strat}).
    \item \textbf{Validity}: If an  action $a(o)$ applies in a state $s$ in the traces,   $Q(x,y,z)$ must be  satisfiable with $x=o$  in $s$. 
    \item  \textbf{Maximality}: $n$ is maximal; i.e., no   other determined variables can be pushed into $z$, and no two variables $z^i$ can be merged into one (i.e.,
         in some action application they denote different objects).
    \end{enumerate}

        The maximality condition is needed so that   the task of learning  action preconditions and effects can be decoupled. Indeed, we
    learn first preconditions, and then, from them,  the action effects. These  effects  may use some of the $z$ variables ``found'' in the first
    step, but not necessarily all of them.
    The query $Q(x,y,z)$ provides  a \emph{referring expression} for each of the variables $z_i$, which  selects a unique value (grounding)
    for $z_i$ in any state $s$ where a ground action $a(o)$ applies. Due to the stratification, the implicit  value for $z_i$ in $a(o)$
    may depend on the values   of the variables $z_j$ preceding $z_i$ in the ordering, but these values are also determined by the binding $x=o$ for $a(x)$ in $s$.
    In this process,  expressions that refer to the same constant in all the states  drawn from the  same instance are discarded,
    as they do not truly denote functions of the state (for example, ``the top-left corner cell in a grid'').
    
    \begin{algorithm}[t]
      \caption{\textsc{test}: checks for unique grounding of $z_{i+1}$ in $Q_{i+1}(x,y,z^{i+1})$ over  all $s$ where $a(o)$ applies, given unique grounding $\sigma$ for
        $z_1, \ldots, z_i$ in $Q(x,y,z^i)$}
    	\label{alg:check_query}
    	\begin{algorithmic}
    		\State \textbf{Input:}  $Q(x,y,z^i) = \bigwedge_{j=1}^{i} Q_j(x,y,z^j)$ \Comment{Valid Query $Q$}
    		\State \textbf{Input:}  $Q_{i+1}(x,y,z^{i+1})$ \Comment{Extension of $Q$}
    		\State \textbf{Input:} AS $ = \{(a(o), s)\}$  \Comment{Relevant state-action pairs}
    		\\
    		\Function{test}{$Q(x,y,z^i), Q_{i+1}(x,y,z^{i+1}),$ AS}
    		
    		\State $Q(x,y,z^{i+1})\gets Q(x,y,z^i)\land Q_{i+1}(x,y,z^{i+1})$  
    		\State unique $\gets$ true
    		\State $V = \{x_1,...,x_n,z_1,...,z_i\}$  \Comment{Variables in $x$, $z^i$}
    		\For{$(a(o), s) \in$ AS}
    		\LineComment{\textit{get $\sigma$ for $z^i$ from $Q(x,y,z^i)$ in $s$ and $x=o$}}
    		\State $\Sigma \gets \text{get\_assignments}(Q(x,y,z^{i+1}), a(o), s,\sigma)$		
    		\If{$|\Sigma| = 0$}
    		\Return \emph{Not-Valid}
    		\ElsIf {$|\Sigma| \geq 2$} unique $\gets$ false
       		\EndIf
       		\LineComment{\textit{Else $|\Sigma|=\{{\sigma}\}=1$ and unique keeps its value}}
    		\EndFor
                \If{ $\neg$ unique}
    		\Return \emph{Not-Determined}
       		\ElsIf{$\exists v \!\in\! V: \forall a(o),s \!\in\!AS: \sigma(z_{i+1})\!=\!\sigma(v)$}
       		\State \Return \emph{Subsumed} 
       		\Else \ \Return \emph{Valid}
    		\EndIf
    		\EndFunction
    	\end{algorithmic}
    \end{algorithm}
    
    \Omit{
 Also we do not add arguments of queries, when the object that is obtained for an argument is in every state the same.
This is because we assume that this is information that is not dependent on the action that is applied but on the instance of the domain.
An example can be seen in the npuzzle where queries exits that can infer the corners of the grid in which the tiles are placed.
These queries yield in any state, where for example an action up is applied the same object.
}
    
    The learning algorithm, that we call \synth, ``synthesizes'' the query $Q(x,y,z)$ one subquery
    $Q_i(x,y,z^i)$ at a time, with  each subquery being constructed one lifted atom $q_{i,j}(x,y,z^i)$
    at a time too, where a \textit{lifted atom} is an atom whose arguments are free variables from $x$, $y$ and $z^i$. 
    More precisely, \synth  constructs  sequences of atoms  of the form
    
        \begin{eqnarray*}
     q_{1,1}(x,y,z^1), \ldots, q_{1,n_1}(x,y,z^1) \ ; \\
      \  q_{2,1}(x,y,z^2), \ldots, q_{1,n_2}(x,y,z^2) \ ; \\
      \  \ q_{m,1}(x,y,z^m), \ldots, q_{1,n_m}(x,y,z^m)
    \end{eqnarray*}

\noindent  where $q_{i,j}$ denotes predicates observed in the traces, 
and $z^i$ contains the variable $z_i$, possibly the  $z$ variables up to $z_i$,
and no variable $z_j$, $j > i$. The $Q_i(x,y,z^i)$ expressions correspond to the conjunction of the lifted  atoms $q_{i,*}(\cdot)$,
and the conjunction of the expressions $Q(x,y,z^i) = Q_1(x,y,z^1)$, \ldots, $Q_i(x,y,z^i)$ must determine single values for the variables $z_j$ for $j \leq i$
and a given value for $x$. 

    The generation of the ordered atom sequences $Q_{i}(x,y,z^i)$ that make up the query $Q(x,y,z)$
    for   satisfying  Conditions 1--3  (stratification, validity, and  maximality) is computed ``greedily'',
    as all the successful maximal queries  end up  denoting the same unique tuple of objects.\footnote{The algorithm can also be seen as
    a dynamic programming algorithm that builds the subqueries $Q_i(x,y,z^i)$ sequentially.} 
    For this, only sequences of atoms are considered that preserve stratification and validity.
    The sequence becomes invalid when the resulting partial query $Q(x,y,z^{i+1})$ becomes unsatisfiable with $x=o$
    in some state $s$ of the traces  where the action $a(o)$ applies. 
    In addition,  atoms  $q_{i+1,*}(x,y,z^{i+1})$ are   considered in  the  sequence  only after  the preceding  variable $z^i$
    has been   determined by the formula $Q_i(x,y,z^i)$ built so far (stratification). Finally, the computation finishes with
    the query   $Q(x,y,z) = Q_1(x,y,z^1), \ldots, Q_n(x,y,z^n)$ when the atom sequence representing the query cannot be extended with an
    additional variable $z_{n+1}$ which is not equivalent to the explicit $x$ variables,  or the  implicit $z_i$ variables, $i \leq n$. 
	
    In short,  starting with $i=1$, \synth   tries  to append a new  atom $q_{i,j}(x,y,z^i)$ to the sequence prefix (initially empty),
    while preserving  validity; i.e., the resulting formula  $Q(x,y,z^i)$ must remain  satisfiable with $x=o$ when $a(o)$ applies in a state $s$. 
    This process continues until a sequence of atoms is found for which the  grounding of variable $z_i$ becomes unique for each action grounding $a(o)$ and state in which the action is applied.   At that point, we have the subquery $Q_i(x,y,z^i)$ and all the ones preceding it.
    
    \begin{algorithm}[t]
    	\caption{\textsc{expand}: Extends $Q(x,y,z^i)$ with lifted atoms $p(w)$  to determine one more variable $z_{i+1}$}
    	\label{alg:find_expansion}
    	\begin{algorithmic}
    		\State \textbf{Input:}  $Q(x,y,z^i) = \bigwedge_{j=1}^{i} Q_j(x,y,z^j)$ \Comment{Valid query Q}
    		\State \textbf{Input:} AS $ = \{(a(o_i), s_i)\}_{i=1}^n$  \Comment{State-action pairs}
    		\State \textbf{Output:} $Q_{i+1}(x,y,z^{i+1})$ \Comment{Valid extension of Q}
    		\State
    		\Function{Expand}{$Q(x,y,z^i), \text{AS}$}
    		\State $Q \gets \{ p(w) \mid p \in P, w\in \{ x, y, z^{i+1} \}$, $z_{i+1} \in w$ $\}$
    		\State $Q_0 \gets \{q_d\}$ \Comment{$q_d$ is dummy query $true$}
    		\While{$Q_0 \not= \emptyset$}
    		\State $Q_{next} \gets \{\}$
    		\For{$q \in Q_0$}
    		\For{$q' \in (Q \setminus q)$}
    		\State $q'' \gets q \land q'$
    		\State $\text{res} \gets \textsc{test}(Q(x,y,z^i), q'', AS)$
    		\If{$\text{res} = $ \emph{VALID}}
    		\State \Return $q''$
    		\ElsIf{$\text{res} =$ \emph{Not-Determined}}
    		\State $Q_{next} \gets Q_{next} \cup q''$
    		\EndIf
    		\EndFor
    		\EndFor
    		\State $Q_0  \gets Q_{next}$		
    		\EndWhile
    		\State \Return \emph{Maximal}  \Comment{$Q(x,y,z^i)$ can't be extended}
    		\EndFunction
    	\end{algorithmic}
    \end{algorithm}

    The two main routines of  \synth are displayed in Algorithms  \ref{alg:check_query} and \ref{alg:find_expansion}.\footnote{The  actual code is  equivalent but more efficient.}
    The first, \textsc{test} procedure, gets as input a query $Q(x,y,z^i)$, a new subquery $Q_{i+1}(x,y,z^{i+1})$, and a set $AS$  of action-state pairs $(s,a(o))$
    where an instance $a(o)$ of action $a$ applies in state $s$ in the given traces, and checks whether the query that conjoins the two makes the $z_{i+1}$ variable
    determined in all such states. For this, it computes the set of groundings $\Sigma$ of the variables in $z^{i+1}$ that satisfy this conjunction for $x=o$,
    assuming that  there is a single grounding $\sigma$ for the variables $z_1, \ldots, z_i$ in $z^i$ that satisfies $Q(x,y,z^i)$ given $x=o$. 
    If $|\Sigma|=0$, it returns that the query extension is not valid, if $|\Sigma|>1$, that it is not determined (no unique grounding),
    and if $|\Sigma|=1$, it returns that it is valid, if the new variable $z_{i+1}$
    does not have the same denotation as a variable in $x$ or $z^i$ over all state-action pairs $(a(o),s)$ in $AS$. Otherwise, the procedure returns that the query extension is subsumed. 

	Algorithm~\ref{alg:find_expansion} displays the procedure \textsc{expand},  which is the heart of \synth: it incrementally refines a precondition query
	$Q=Q(x,y,z^i)$,  which initially  contains only the \emph{dummy query} $q_d$, that is  always  true and does not involve (constraint) any variables,
	by conjoining it with lifted  atoms $q$, one at a time,  until these atoms jointly form the new query component $Q_{i+1}(x,y,z^{i+1})$ that  determines
	the unique  grounding of the variable $z_{i+1}$.
  For this, the \textsc{test} procedure is called with $Q$, $q$, and the relevant set of state-action pairs $AS$
 (instances $a(o)$ of $a$ in the traces and the states $s$ where are applied).
  The atom $q$ can be used to expand the query only if \textsc{test} returns ``valid'', and in this case the exapansion $Q_{i+1}(x,y,z^{i+1})$ is returned. 
 	If the algorithm \textsc{test} returns ``not-determined'', the algorithm \textsc{expand}  will expand $q$ in the next iteration.
 	When the subquery is ``non-valid'' or ``subsumed'' we known that it cannot be used to expand $Q(x,y,z^i)$.

 \medskip
 
\noindent     The algorithm \synth  is complete in the following sense:

    \begin{theorem}
      Let $\phi(x,y,z)$ be the  precondition of the action $a(x)$ in the hidden, stratified domain,
     and let $z_i = c_i^{s,o}$  be the value of variable $z_i$ in the state $s$ of the traces where the action $a(o)$ applies, $i=1, \ldots, n$. 
     Then \synth returns a query $Q(x,y',z')$ from the traces with $z'=(z'_1, \ldots, z'_m)$
     where  each $z_i$ variable in $z$ is captured by a  $z'_j$ variable in $z'$;
     namely, there is a function $\rho(z_i)=z'_j$ such that in all such   states $s$,
     $Q(x,y',z')$ is satisfiable in $s$ with $x=o$, and in all the  satisfying groundings $\sigma$,
     $\sigma(z'_j)=c_i^{s,o}$.
    \end{theorem}

    The theorem doesn't imply  soundness in the sense that groundings of $Q(x,y,z)$ that satisfy $x=o$ in the  state $s$ of the traces, imply that  the ground action $a(o)$
    must be  applicable in $s$. Indeed, for this, the precondition query $Q(x,y,z)$ needs to be extended with other  atoms as shown below.
        
    In the implementation of \synth, the lifted atoms $q_{i,j}(x,y,z^i)$  are ordered lexicographically, so that a single ordering of the atoms  is considered in
    the construction of the subquery $Q_i(x,y,z_i)$. Moreover, these atoms are  represented by atom patterns, following  an idea from \sift \cite{sift}.
    Namely, a pattern  like      $p(\_,2,i_z,1,\_,1_z)$  is used  to represent  the  lifted atom $p(y_1,x_2,z_i,x_1,y_2,z_1)$ in $Q_i(x,y,z^i)$ where $i > 1$.

\subsection{Learning the Extra Preconditions and Effects}

The precondition formula $Pre_L(a(x))$ of  the \stripsp action $a(x)$ in the   learned domain $D_L$  conjoins  the query $Q(x,y,z)$ obtained from the traces
with a formula  $Q'(x,y,z)$ that does  not refine the binding of the $z$ variables,  but which contains all lifted atoms that are true in all the states $s$ where an action $a(o)$ applies.
This ensures that if $Pre_L(a(x))$ is satisfiable with $x=o$ in a trace, then the precondition $Pre(a(x))$ of the  action $a(x)$ in the hidden domain $D$ will
also be  satisfiable with $x=o$, and hence that if the ground action $a(o)$ applies in the state $s$ of a trace in the learned domain $D_L$, 
it also applies in $s$ in the hidden domain. The reverse implication is also true, and if $a(o)$ applies in a state $s$ of a trace in the hidden
domain $D$, it also is applicable in the same state $s$ in the learned domain $D_L$. Otherwise, the action $a(x)$ would feature a precondition
in the learned domain that is not true in a state $s$ of the traces where an action $a(o)$ is done, and hence  the precondition in $D_L$ would
not be valid over the traces.

The formula $Q'(x,y,z)$ to augment the precondition query $Q(x,y,z)$ in $Pre_L(a(x))$ is obtained as the conjunction of the lifted atoms $p(x,y,z)$
whose arguments are variables $x_i$ from the lifted action $a(x)$, $z_i$ variables from $z$, and $y_i$ variables from $y$. 
The $y_i$ variables
can only appear once in these lifted atoms and they are assumed to be existentially quantified. The $x_i$ and $z_i$ variables are free,
and in a state $s$ where an action $a(x)$ is applied, their values are determined by the action arguments and the 
the precondition query $Q(x,y,z)$.

A lifted atom $p(x,y,z)$ is a \emph{valid precondition} of $a(x)$ given the traces $T$, and hence pushed into $Q'(x,y,z)$ and $Pre_L(a(x))$,
if the formula $\exists y'. p(x,y',z)$ is true in all the states  $s$ of $T$ where an instance $a(o)$ of $a$ is applied, 
provided  the binding  $x=o$ for $x$  and the unique grounding for  $z$  determined by the precondition query $Q(x,y,z)$.

Likewise, a lifted atom $p(x,z)$ is a \emph{valid positive (resp. negative) effect} of $a(x)$ given the traces $T$, and hence pushed into $Add(a(x))$
(resp. $Del(a(x))$), if the formula $p(x,z)$ is false (resp. true) in a state $s$ of $T$,  where an  action instance $a(o)$ of $a$ is  applied,
and  true (resp. false) in the resulting state $s'$, provided the binding 
 $x=o$ for $x$ and the unique grounding for  $z$ determined by the precondition query $Q(x,y,z)$.

\subsection{Properties}

If $D_L$ denotes the learned model, namely, the action schemas $a(x)$ for the actions appearing in the traces $T$ with their learned
preconditions and effects, it can be shown that  $D_L$ is  equivalent to the hidden domain $D$  provided that the set of traces $T$ is rich enough
and that $D$ is stratified (i.e., the unique  grounding of the $z$ variables can be determined one $z_i$ variable at a time):

\begin{theorem}
  For a suitable finite set of traces $T$ from the hidden (stratified) domain $D$, the learned domain $D_L$ is equivalent to $D$, meaning that
  the state-action traces resulting from any instance $P=\tup{D,I}$ of $D$, are traces of $P'=\tup{D_L,I}$, and vice versa.
\end{theorem}

\begin{proof} (Sketch)   \synth learns queries $Q_i$ for each of the hidden $z_j$ variables in $D$ such that the hidden preconditions
  $\phi(x,y,z)$ and $Q_i$ pick up the same denotations. 
  The resulting learned  preconditions in $D_L$  may be different than in $D$, and the indices $i$ and $j$ may be different too, but the 
  function $f_{a,s}(x)$  represented by both, that  defines the unique grounding of the $z$ variables, would  be the same. 
 Also, by construction  all the  preconditions in $D$   are expressible with the  $x,y,z$ variables, and hence will be captured in the learned domain $D_L$,
 which may include other preconditions too. 
 Similar to preconditions, all effects are expressible with $x,z$ variables, and since the domain is assumed to be well-formed, all effects of the hidden domain are captured by the learned domain $D_L$ too.  Finally,  any invalid precondition will be rendered invalid through a single state in a  trace, and there is
  a finite number of such invalid preconditions that can be constructed. 
\end{proof}

\subsection{Negation}


The extension of \synth for  learning  \emph{negated preconditions} is convenient and direct, and
is implemented  in the  algorithm. The only change is that in negated lifted atoms $\neg p(x,y,z)$,
the $y$ variables are interpreted as \emph{universally} quantified; namely, as  $\forall y. \neg p(x,y,z)$.
As a result, a precondition $\neg p(x,y,z)$ of an action $a(x)$ is true in a state $s$ with $x=o$ and
$z=o'$, if the formula $\forall y. \neg p(x,y,z)$ is true in $s$ with the bindings for the $x$ and $z$
variables provided by $o$ and $o'$. 

For example, in the \puzzle,  the position $z_1$  of  the blank can be obtained  without the predicate $blank(x_i)$
through  the negated precondition $\neg at(y_1, z_1)$  which stands for the formula $\forall y_1. \neg at(y_1, z_1)$,
as this formula is true only when $z_1$ represents the unique cell without a tile.


\section{Experiments}

We have tested \synth   over a number of existing \strips domains $D'$.
For this, we converted  these domains  into \stripsp domains $D$
by moving the explicit action arguments in $D'$ that are determined 
into implicit $z$ arguments in $D$. On average, as we will see,
this reduces the number of (observed) action arguments  during  training
by half. Then a single random state-action trace  is sampled from a 
large instance $P=\tup{D, I}$ from $D$ to learn  the domain $D_L$. 
The hidden and the learned domains are then compared over
other test instances in a  verification phase.
We provide further details of the set up below, along with the results.
A finer grained analysis of the results can be found in the next section.
The experiments have been run on Intel(R) Xeon(R) Platinum 8160
CPUs running at 2.10GHz and the data and code are publicly available \cite{synth_implementation}.

\medskip

\noindent \textbf{Domains:} The domains are the ones considered  in the \sift paper \cite{sift}: 
Blocks with 3 and 4 operators, Delivery, Driverlog, Grid, Ferry, Gripper, Hanoi, Logistics,
Miconic, Sliding-tile puzzle, Sokoban.  The Sliding-tile puzzle is
in two versions: one with separate  $x$ and $y$ coordinates denoted as $n$-puzzle, and the other with cells,
denoted as $c$-puzzle. Sokoban-Pull is a 
variation of Sokoban,  adding one action schema for a pull-action to make the resulting domain
\emph{dead-end free}.
Dead-ends present a potential problem in the  generation of data, as random traces are
often trapped in  parts of the state space.

\medskip

\noindent \textbf{Translation into \stripsp:}  The traces are not generated from these  \strips domains $D'$
but from their  \stripsp translation  $D$ where  some of the explicit action arguments in \strips
are pushed into  implicit $z$  arguments. The algorithm for doing this translation automatically
is a simplification of the query learning component of \synth, as the preconditions are given.
In order to determine if an argument  variable $x'_i$ in a \strips\ action $a'(x')$
from $D'$ is determined by  other arguments and hence can become a $z$ variable in the  encoding of  the equivalent action $a(x)$ in
$D$, we  check  if in all states of  traces drawn from $D'$ where the \strips action $a'(o')$ is applied, the value of the argument $o'_i$
is unique given the preconditions.
The preconditions and effects  of the  \strips action $a'(x')$ and the resulting \stripsp action $a(x)$ are the same except for the renaming
of the variables $x'_i$ into  $x$  and $z$ variables.
There are no (existentially quantified)  variables $y$ in $D$, but they can  appear in the learned domains $D_L$.

\medskip
\noindent \textbf{Data generation:} For each \strips domain $D'$, we pick an instance $P=\tup{D,I}$ from the \stripsp translation,
and generate a long random trace with up to ten thousands of state-action pairs.\footnote{Traces from multiple instances could have been used too.}
In some domains, small instances and  short   traces suffice to learn the domains correctly, in other cases,
larger instance and/or longer traces are needed.

\medskip

\noindent \textbf{Verification:} The correctness of the learned domains $D_L$ is  assessed by sampling a number of reachable states $s$
and ground actions $a(o)$ in instances  $P=\tup{D,I}$ where $D$ is the hidden domain. The applicability and effects of these
actions are then compared in $D$ and $D_L$.  A 100\% verification rate indicate a full match. For testing also the translation from \strips to \stripsp,
the \strips actions of the original domain that are applicable in $s$ are also considered, and they must  map to
the same set of successor states. This  comparison with the original \strips domain  is not done however
in the experiments. 

\medskip

\noindent \textbf{Results:} Table~\ref{table:results} shows the results of the experiments. The columns on the left show the domains,
the number of objects in the instance used to generate the single  trace per domain  (\#O), the length of the trace measured as the number of state-action pairs (\#L),
and the  total number of  action arguments in  the original  \strips domain $D'$ ($|x'|$),
and the reduced number of action arguments in the \stripsp translation $D$ used to generate the traces 
($|x|$). This is followed by the sum of this  number $|x|$ and  the total  number of implicit action arguments $z_i$ 
learned along with their queries  ($|x|+|z|$). The following columns show the total number of
explicit but determined  \strips arguments $x'$ that the learned $z$ variables fail to capture ($|x'/z|$), \footnote{
In the translation from \strips $D'$ into \stripsp $D$, some $x'_i$ variables in the \strips actions $a'(x')$
are pushed into $z_j$ variables in the \stripsp actions $a(x)$ of $D$ (this is why $|x| < |x'|$ in general).
These $z_j$ variables and their queries have to be learned in $D_L$, and then $z_j$ captures $x'_i$
if in all the states $s$ where $a(x)$ applies, $z_j$ and $x'_i$ represent the same object.}
and the total  number of $z$ variables  learned  which do not capture any $x'$ variable ($|z/x'|$).
This is extra information learned; namely, valid query expressions that define  functions $f_{a,s}(x)$  over the traces
that are just not  used in $D_L$.  This is followed by the total learning time (T),
and the verification data: number of objects used in the  verification instance (\#O$_v$), number of action-state pairs
tested (\#SA), the verification time (T$_v$) and  the score (\%V). The data on the top part of the figure is about experiments where
the full  \stripsp states in $D$, which are equal to the full \strips states in $D'$, are used in the traces. 
The data on the bottom part is about the  experiments where  all atoms involving  selected
predicates were removed  from the states in the traces (incomplete states).


\medskip

\noindent \textbf{Analysis:} In all domains, \synth learns domains $D_L$ that verify 100\%. The learning times run from a few seconds,  to 1695 seconds in
Driverlog. The times grow  with the number of domain predicates, their arities, and the length of the traces. Both the length of the traces and the size
of the instances used  to generate the traces were selected  so that \synth outputs the  correct domains. This doesn't happen if the traces are too short
or the instances are too small. The domains that required the longest traces are Driverlog and Grid. At the same time, the $n$-puzzle provides a good illustration of the
size of the instances required for correct learning. The $c$-puzzle  uses cells and the  domain is learned correctly from traces of the 4x4 $c$-puzzle.
The $n$-puzzle, on the other hand,  uses separate $x$ and $y$ coordinates instead of cells,
and requires traces from the larger  5x5 $n$-puzzle, as  smaller instances resulting in invalid
referring expressions ($z$ variables and their queries).

Table \ref{table:results} shows that  the actions in the traces contain on  average half of the arguments of the original \strips actions (i.e.,
the column  $|x'|$, which expresses the total number of \strips action arguments, is  on average twice the value of the column $|x|$,
which captures the total number of \stripsp action arguments used in the traces). More details about this below.
The learned queries capture indeed  all  the ``redundant'' (determined) \strips action arguments (column $|x'\!\setminus\!  z|$)
and more (column $|z \!\setminus\!  x'|$). For instance, in Gripper,   there are just two rooms and two grippers, and $z$ variables
are learned that  pick ``the other room'' and ``the other gripper'', which are not used in the learned preconditions or effects of the actions.

The bottom part of the table shows the results when learning from \emph{incomplete states} in four  domains where all atoms involving  selected
predicates were removed from the traces.
The  $blank$ predicate is removed in the $c$-puzzle,\footnote{Interestingly, 
in  the $n$-puzzle, which is like the $c$-puzzle but with separate $x$ and $y$ coordinates,
the same predicate cannot  be removed, as  the limited form of existential quantification
in action preconditions in our \stripsp fragment is not expressive enough to recover it.}
%
%
the \emph{on-table} and \emph{clear} predicates in Blocks, the \emph{in-lift} predicate in Miconic, and the \emph{on} predicate in Ferry.
These predicates can  be defined in terms of the other predicates and thus do not provide   information about the state of the world,
but about the atoms needed to get an  \strips encoding. In \stripsp, these predicates are learned as  existentially quantified
preconditions and  are not needed in the traces.

\begin{table*}[t]
\setlength{\tabcolsep}{5.6pt}\centering
    \begin{tabular}{lcccc|ccccc|crcc}
\multicolumn{5}{c}{Data} & \multicolumn{5}{|c}{Learning} & \multicolumn{4}{|c}{Verification} \\
Domain      & $\#\textrm{O}$ & $\#\textrm{L}$  & $|x'|$ & $|x|$   & $|x|\! +\! |z|$ & $|z|$ & $|x' \!\setminus\! z| $& $|z \!\setminus\! x'|$ & $\textrm{T}$ & $\#\textrm{O}_\textrm{V}$ & $\#\textrm{S}_\textrm{V}$\ & $\textrm{T}_\textrm{V}$ & $\%\textrm{V}$ \\
\midrule
blocks3     & $5$            & $250$           & $7$    & $5$     & $7$             &$2$    & $0$ & $0$ & $0.82 s$    & $6$  & $1200$ & $80.56 s$  & $100 \%$ \\
blocks4     & $5$            & $250$           & $6$    & $3$     & $6$             &$3$    & $0$ & $0$ & $1.43 s$    & $6$  & $1600$ & $105.04 s$ & $100 \%$ \\
delivery    & $16$           & $1000$          & $9$    & $5$     & $9$             &$4$    & $0$ & $0$ & $101.4 s$   & $24$ & $1200$ & $246.69 s$ & $100 \%$ \\
driverlog   & $63$           & $10000$         & $19$   & $10$    & $19$            &$9$    & $0$ & $0$ & $1695.32 s$ & $148$& $2400$ & $3757.22 s$& $100 \%$ \\
ferry       & $8$            & $100$           & $6$    & $2$     & $6$             &$4$    & $0$ & $0$ & $0.96 s$    & $10$ & $1200$ & $68.78 s$  & $100 \%$ \\
grid        & $50$           & $10000$         & $13$   & $4$     & $17$            &$13$   & $0$ & $4$ & $584.62 s$  & $48$ & $2000$ & $1264.47 s$& $100 \%$ \\
gripper     & $10$           & $500$           & $8$    & $3$     & $11$            &$8$    & $0$ & $3$ & $2.52 s$    & $12$ & $1000$ & $96.14 s$  & $100 \%$ \\
hanoi       & $8$            & $200$           & $3$    & $2$     & $3$             &$1$    & $0$ & $0$ & $1.5 s$     & $10$ & $400$  & $41.58 s$  & $100 \%$ \\
logistics   & $35$           & $1000$          & $13$   & $7$     & $20$            &$13$   & $0$ & $7$ & $28.25 s$   & $45$ & $1600$ & $377.5 s$  & $100 \%$ \\
miconic     & $9$            & $600$           & $8$    & $2$     & $8$             &$6$    & $0$ & $0$ & $2.78 s$    & $12$ & $1600$ & $104.14 s$ & $100 \%$ \\
$n$-puzzle  & $34$           & $1000$          & $16$   & $0$     & $16$            &$16$   & $0$ & $0$ & $498.28 s$  & $34$ & $1600$ & $1123.36 s$& $100 \%$ \\
$c$-puzzle  & $49$           & $500$           & $12$   & $0$     & $12$            &$12$   & $0$ & $0$ & $147.19 s$  & $49$ & $1600$ & $569.38 s$ & $100 \%$ \\
sokoban     & $30$           & $1000$          & $5$    & $2$     & $5$             &$3$    & $0$ & $0$ & $20.7 s$    & $30$ & $800$  & $72.45 s$  & $100 \%$ \\
sokopull    & $25$           & $600$           & $8$    & $3$     & $8$             &$5$    & $0$ & $0$ & $10.67 s$   & $30$ & $1200$ & $80.51 s$  & $100 \%$ \\
\midrule
blocks3$^-$ & $5$            & $250$           & $7$    & $5$     & $7$             &$2$    & $0$ & $0$ & $0.73 s$    & $6$  & $1200$ & $56.8 s$   & $100 \%$ \\
ferry$^-$   & $8$            & $100$           & $6$    & $2$     & $6$             &$4$    & $0$ & $0$ & $0.54 s$    & $10$ & $1200$ & $63.71 s$  & $100 \%$ \\
miconic$^-$ & $9$            & $600$           & $8$    & $2$     & $8$             &$6$    & $0$ & $0$ & $2.0 s$     & $12$ & $1600$ & $89.34 s$  & $100 \%$ \\
$c$-puzzle$^-$& $49$           & $500$           & $12$   & $0$     & $12$            &$12$   & $0$ & $0$ & $87.67 s$   & $49$ & $1600$ & $488.87 s$ & $100 \%$

\end{tabular}
    \caption{
        Table of results when learning a \stripsp domain $D_L$ from a trace of of length $\#\textrm{L}$ from  hidden domain $D$ derived from a  \strips   domain $D'$ with $\#\textrm{O}$ objects.
        $|x'|$ is the total number of action  arguments in $D'$, $|x|$ is the total  number of action  argument in $D$, $|z|$ is the number of learned implicit action arguments in $D_L$
        and $|x|\! +\! |z|$ is their sum.         $|x' \!\setminus\! z|$ is the number of action  arguments   in $D'$ not captured in  $D_L$,
        while $|z \!\setminus\! x'|$ is the number of implicit action arguments in $D_L$ that are not in $D$ (they are not incorrect, just not used). 
        $\textrm{T}$ is the total  time  to learn $D_L$, $\textrm{T}_\textrm{V}$ is the time  to verify $D_L$,  $\#\textrm{S}_\textrm{V}$ is the number of sampled state-action pairs in
        the  verification. All numbers are averages over 10 runs as traces are random. $\%\textrm{V}$ is the success rate of the verification.
        The rows in the  second block show the  experiments using incomplete  \strips states in the traces  where atoms involved selected predicates are  removed. 
    }
    \label{table:results}
\end{table*}

\section{Fine-grained analysis}

We provide an  analysis of the learned queries that  account for the missing  action arguments in the traces, 
that  become the  implicit action arguments in the learned domains.


\begin{itemize}
\item \textbf{Gripper:} The \strips actions are $move(r,r'),$ $grab(b, r, g)$ and $drop(b, r, g)$. The actions in the \stripsp traces
  are $move()$, $grab(b,g)$, and $drop(b)$. \synth learns the precondition queries that capture the omitted arguments which are involved
  in preconditions and effects. For example, the current and next rooms  in $move$  are captured by the learned queries $Q_1 = \{at(z_1)\}$ and $Q_2 = \{\neg at(z_2)\}$.
  These queries also capture the rooms for the actions $pick$ and $drop$.
  The missing gripper of the action \emph{drop}  is obtained by the  learned query $Q_3 = \{carry(b, z_3)\}$,
  and actually,  the other gripper is also identified through the  query  $Q_4 = \{\neg carry(b, z_4)\}$.
  One can see that for $pick$ and $drop$,  the referring expression for the ``room in which the robot is not located'' is derived,
  while  for $drop$, the expression for the ``gripper in which the ball is not held'' is derived. 
  These  referring expressions and their corresponding $z$ variables, however, are not used in  the learned preconditions or  effects.
  
\item \textbf{Sokoban:} The \strips actions are  $move(c_1, c_2)$ and  $push(c_1,c_2,c_3)$,
  which translate to the  \stripsp actions $move(c_2)$ and $push(c_2)$. 
						For both actions the location of the agent $c_1$ is  obtained by $z_1$ with  query $Q_1 = \{at(z_1)\}$, 
                                                while  $c_3$ is obtained by $z_2$ with query 
						$Q_2 = \{ adjacent_2(z_1, z_2), adjacent(c_2, z_2) \}$.
                                                
\item \textbf{c-Puzzle:} The \strips actions in  $c$-puzzle are $up(c_1,c_2,t)$, $down(c_1,c_2,t)$, $left(c_1,c_2,t)$, and $right(c_1,c_2,t)$.
  In \stripsp, the resulting  actions  take no  arguments. Focusing on the first action,  the position of the ``blank'' $c_2$ is
  picked up  by a  $z_1$ variable with query $Q_1 = \{\neg at(y_1,z_1)\}$, while   the next position of the ``blank'' $c_1$
  is picked up  by  $z_2$ with query $Q_2 = \{above(z_1, z_2)\}$,
  and the tile $t$ that is moved is picked up by $z_3$ with query      $Q_3 = \{ at(z_3, z_2)\}$.
\Omit{     the tile $t$ is picked by the $z_2$ variable with query 
       From this, for $up$ the position in which the blank is pushed is obtained from $Q_2 = \{above(z_2, z_1)\}$ and the tile in this cell with
       , similar for $left,$ $right,$ $down$.}
The predicate $blank$ is not used in a query
and also is  not needed in a precondition, since
for example $up$ is applicable, if there is a cell above the $blank$ which is the case iff $Q_2$ is satisfied.
						Therefore, the predicate $blank$ does not need to be contained in the states such that the domain can be learned.
                                                
\item \textbf{Blocks World 3:} The actions  in \strips are  $stack(b_1, b_2)$, $unstack(b_1, b_2)$, and $move(b_1,b_2,b_3)$. 
  In  \stripsp, these actions reduce to $stack(b_1, b_2)$, $unstack(b_1)$, and $move(b_1,b_3)$. 
  The argument $b_2$ is picked by a variable $z_1$ with   query $Q_1 = \{ on(b_1,z_1)\}$.
  In addition, the atoms  $on\_table(x)$, $clear(x)$ are not needed in the state traces, as they can be learned as  the negated preconditions
  of the form  $\neg on(x,y)$ and  $\neg on(y,x)$ respectively, that stand for   $\forall y. \neg on(x,y)$ and  $\forall y. \neg on(y,x)$.
\end{itemize}

\section{Conclusions}


The problem of learning lifted \strips model from action traces alone
has been recently solved by the \sift algorithm, a follow up to 
\textsc{locm}. The limitation is that the actions in the traces
must come from a hidden \strips domain and include all the arguments. 
This means for example that to represent the action of unstacking 
a block $x$ in a trace, the location of the block $x$ must be
conveyed as an extra argument. In this work, we addressed a new
variant of the model learning problem which is  more realistic,
and closer to the settings used in model-based reinforcement learning
where actions reveal a minimal number of arguments. The problem is formulated
and solved as the task of  learning lifted \stripsp models
from \stripsp state-action traces. The resulting algorithm
has a broad and crisp scope, where it is sound, complete, and
scalable, as illustrated through the experiments. 

One question that arises from this work is whether the
proposed methods can be used to learn, for example,
the deterministic dynamics of Atari-like video games. In this setting,
actions have no explicit arguments and  states are represented by
the colors of the cells in a grid (pixels). One difference
to learning the dynamics of the sliding-tile puzzles from
grids (atoms $at(t,c)$)  is that the objects in Atari
can take up many cells. This seems to call for representation
languages that are richer than \stripsp, able to accommodate
definitions (axioms) and partial observability, as 
one does not observe objects directly but cell colors.

\section*{Acknowledgments}

The research has been supported by the Alexander von Humboldt Foundation with funds from the German Federal Ministry for Education and Research,
by the European Research Council (ERC) under the European Union's Horizon 2020 research and innovations programme (Grant agreement No. 885107),
and by  the Excellence Strategy of the Federal Government and the NRW State.



\bibliographystyle{kr}
\bibliography{control}

\end{document}